\documentclass{article}
\usepackage{ijcai21}

\usepackage{times}

\usepackage{soul}
\usepackage{url}
\usepackage[hidelinks]{hyperref}
\usepackage[utf8]{inputenc}
\usepackage[small]{caption}
\usepackage{graphicx}
\usepackage{amsmath}
\usepackage{amsthm}
\usepackage{mathrsfs}
\usepackage{amssymb}
\usepackage{booktabs}
\usepackage{algorithm}
\usepackage{algorithmic}
\usepackage{bm} 
\usepackage{overpic}
\usepackage{subfigure}
\usepackage{color}
\usepackage[T1]{fontenc}
\usepackage[dvipsnames]{xcolor}
\newcommand{\etal}{\textit{et al.} }
\usepackage[final]{changes}
\newtheorem{theorem}{Theorem}

\newtheorem{lemma}{Lemma}

\title{Optimal ANN-SNN Conversion for Fast and Accurate Inference in Deep Spiking Neural Networks}

\author{
	Jianhao Ding$^{1}$
	\and
	Zhaofei Yu$^{1,2,3}$\thanks{Corresponding author}\and
	Yonghong Tian$^{1,3 *}$\And
	Tiejun Huang$^{1,2,3}$\
	\affiliations
	$^1$Department of Computer Science and Technology, Peking University\\
	$^2$Institute for Artificial Intelligence, Peking University\\
	$^3$Peng Cheng Laboratory
	\emails
	djh01998@stu.pku.edu.cn, \{yuzf12,yhtian,tjhuang\}@pku.edu.cn
}

\begin{document}

\maketitle

\begin{abstract}


Spiking Neural Networks (SNNs), as bio-inspired energy-efficient neural networks, have attracted great attention\added{s} from researchers and \deleted{the }industry. The most efficient way to train deep SNNs is through ANN-SNN conversion. However, the conversion usually suffers from accuracy loss and long inference time, which impede the practical application of SNN. In this paper, we theoretically analyze ANN-SNN conversion and derive \deleted{the }sufficient conditions of the optimal conversion. To better correlate ANN\replaced{-}{ and }SNN and get greater accuracy, we propose\deleted{ the} Rate Norm Layer to replace the ReLU activation function in source ANN training, enabling direct conversion from a trained ANN to an SNN. Moreover, we propose an optimal fit curve to quantify the \replaced{fit}{fitness} between the activation value of source ANN and the actual firing rate of target SNN. We show that the inference time can be reduced by optimizing the upper bound of the fit curve in the revised ANN to achieve fast inference. Our theory can explain the existing work on fast reasoning and get better results. The experimental results show that the proposed method achieves near loss-less conversion with VGG-16, PreActResNet-18, and deeper structures. Moreover, it can reach $8.6\times$ faster reasoning\added{ performance} under $0.265\times$ energy consumption of the typical method. The code is available at \color{RubineRed}\textit{\url{https://github.com/DingJianhao/OptSNNConvertion-RNL-RIL}}.

\end{abstract}

\section{Introduction}

As a representative of artificial intelligence methods, deep learning has begun to exceed or approach human performance in various tasks, including image classification, natural language processing, and electronic sports~\cite{he2016deep,brown2020language,berner2019dota}. But this success is at the cost of high energy consumption. Recently, neuromorphic hardware, including TrueNorth, SpiNNaker, Loihi, and so on~\cite{debole2019truenorth,painkras2013spinnaker,davies2018loihi}, is attracting more and more researchers due to their high temporal resolution and low power budget. This kind of hardware runs Spiking Neural Networks (SNNs) instead of Artificial Neural Networks (ANNs).  With unique memory and communication designs, an SNN implemented on SpikNNaker \replaced{can achieve}{achieves} the power consumption of 0.3W for MNIST classification~\cite{stromatias2015scalable}. For object detection, the spiking version of a YOLO model is estimated to be at least 100 times\added{ more} energy-efficient than that on GPU~\cite{kim2020spiking}. Such dedicated hardware and algorithms are highly appealing for mobile applications like autonomous vehicles and other power-limited scenes.

\begin{figure}[t]
	    \includegraphics[width=0.9\linewidth,trim={0 0 0 0},clip]{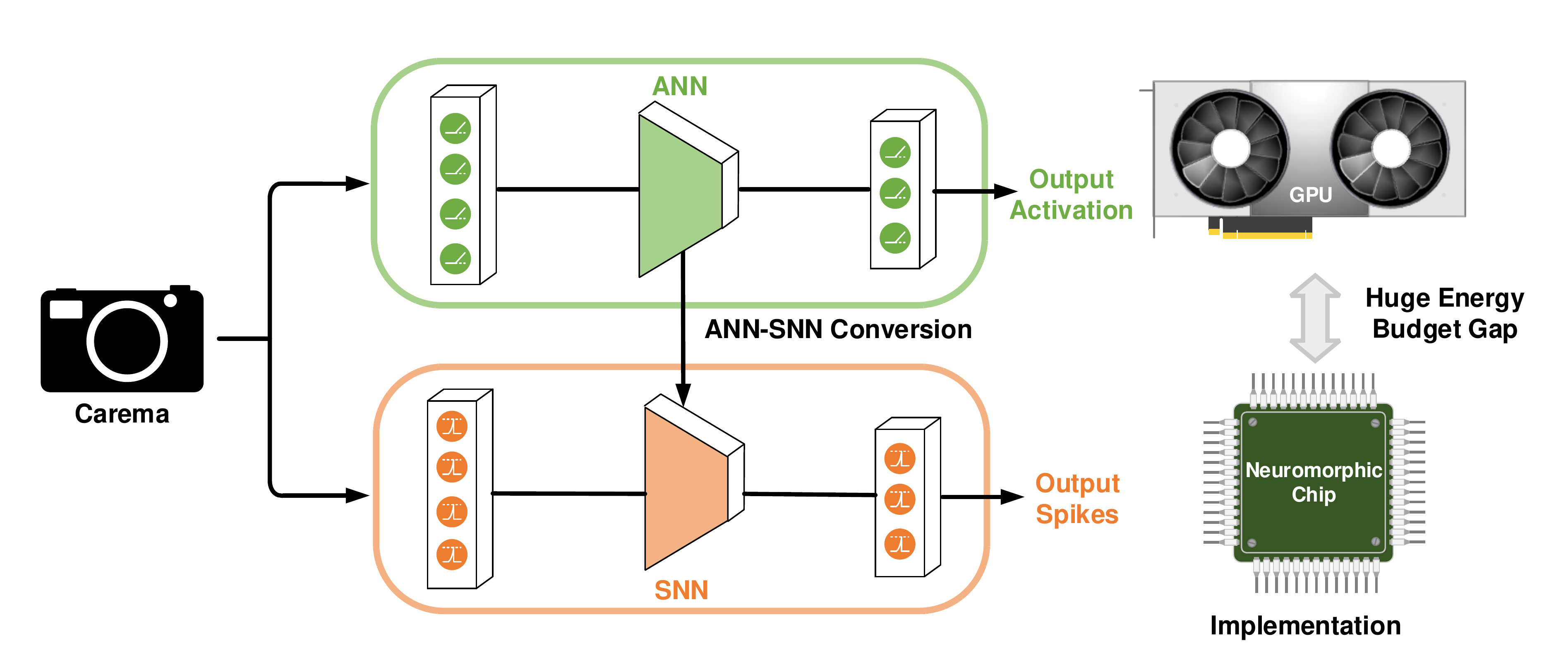}
	    \caption{Illustration of the ANN-SNN converison.}
\label{fig:annsnn1}	
\end{figure}

Nevertheless, training high-performance SNNs is a nontrivial problem. The neurons in SNNs emit discrete spikes, which disables the direct backpropagation training. Up to now, the training algorithms for SNNs can be summarised into four methodologies: supervised backpropagation through time~\cite{wu2019direct,zenke2021remarkable}, unsupervised STDP learning~\cite{kheradpisheh2018stdp,diehl2015unsupervised}, ANN-SNN conversion~\cite{cao2015spiking,diehl2015fast,rueckauer2017conversion}, and other mixture methods~\cite{lee2018training,tavanaei2019bp,rathi2020enabling}. For deep \replaced{SNN}{SNNs} training, ANN-SNN conversion requires less GPU computing than supervised training with surrogate gradients. Meanwhile, it has yielded the best performance in large-scale networks and datasets among methodologies. Therefore, ANN-SNN conversion has become the first choice for deep SNN training, which is also the focus of this paper.

As illustrated in Fig.~\ref{fig:annsnn1}, ANN-SNN conversion is to map the parameters of a pre-trained ANN to an SNN with low accuracy loss. 
Cao \etal~\shortcite{cao2015spiking} started the study of ANN-SNN conversion. They found the equivalence between the ReLU activation and \replaced{the }{a }spiking \replaced{neurons'}{neuron's} firing rate, which is the foundation of later rate-based methods. 
Diehl \etal~\shortcite{diehl2015fast} attributed the performance loss to inappropriate activation of neurons and proposed Weight Normalization methods (model-based and data-based) to scale the \added{ANN }weights. 
Rueckauer \etal~\shortcite{rueckauer2017conversion} gave a detailed theoretical explanation of ANN-SNN conversion and proposed a new reset-by-subtraction neuron to overcome accuracy degradation. They also extended the use of bias and Batch Normalization (BN) and proposed the Max Normalization algorithm\added{ (Max Norm for short)}, which uses maximum\added{ values of} activation as scaling factors. 
Kim \etal~\shortcite{kim2020spiking} suggested to use \deleted{enhanced normalization and }channel-wise normalization for convolutional nets. 
Different from conversion using ANN activation, Sengupta \etal~\shortcite{sengupta2019going} proposed SpikeNorm, which makes use of spiking statistics to set the thresholds. 
\deleted{In a sense, setting the threshold is equivalent to scaling parameters. }
To enable more accurate conversion, Rueckauer \etal~\shortcite{rueckauer2017conversion} further proposed Robust Normalization\added{ (Robust Norm for short)}, where the scaling factor changes from maximum activation value to 99.9\% of activation. Yet this is not the first attempt to manually manipulate the threshold or factor. Cao \etal~\shortcite{cao2015spiking} set the firing thresholds based on spike density. These practices allow the firing rates of some neuron\added{s} to be constant 1. Here we refer to this phenomenon as \emph{spike saturation}. 
However, both the maximum and 99.9\% are rigid, which inspires us to explore a trainable way to achieve low conversion loss.

The conversion methods mentioned above all incur long simulation time when applied to deeper networks and more complicated datasets. That is, converted SNNs need a longer time to rival the original ANNs in precision. This restricts the practical promotion, such as real-time tracking and detection. 
Robust Normalization somewhat mitigates this problem by increasing the firing rates. Spike saturation actually causes the subsequent layers to infer faster. Based on this observation, Han \etal~\shortcite{han2020rmp} started to improve the inference latency by scaling the SpikeNorm thresholds. Han \etal then gave a theoretical analysis of the scale (setting the firing threshold as the expectation of weights $\times$ spikes), but they used the manually set value eventually. Nevertheless, this inspires us that inference latency and parameters can establish associations on the model. A hybrid training scheme also helps. Rathi \etal~\shortcite{rathi2020enabling} realized fewer inference time-steps by conversion-based initialization and spike-timing-dependent backpropagation. Other methods concern coding schemes to achieve fast inference \added{(i.e. shorter inference time)}, including Temporal-Switch Coding~\cite{han2020deep}, FS-conversion coding \cite{stockl2020classifying}. However, simply bypassing rate coding is not that rational. Though rate coding is not the perfect coding scheme, it is partially in line with observations \replaced{in}{of} the visual cortex~\cite{rullen2001rate}. Therefore, fast inference for rate-encoding deep SNNs is an important research direction. But it still lacks instructive principles and theories.

In this paper, we propose an ANN-SNN conversion method that enables high accuracy and low latency. The main contributions of this paper are summarized as follows:
\begin{itemize}
    \item We theoretically analyze ANN-SNN conversion and derive the sufficient conditions of the optimal conversion. Based on this, we propose\deleted{d the} Rate Norm Layer to replace the ReLU activation function in source ANN,  enabling direct conversion from a trained ANN to an SNN. \deleted{without any operation. }This will reduce the potential loss of information caused by normalization.
    \item We propose an optimal fit curve to quantify the fit between the activation value of source ANN and the actual firing rate of target SNN, and derive one upper bound of this convergent curve. We show that based on the Squeeze Theorem, the inference time can be reduced by optimizing the coefficient in the upper bound. These results can not only systematically explain previous findings that reasonable scaling of the threshold can speed up inference, but also give a proper theoretical basis for fast inference research.
    \item We demonstrate the utility of the proposed method with near loss-less conversion in deep network architectures on the MNIST, CIFAR-10, CIFAR-100 datasets. Moreover, it achieves $8.6\times$ faster reasoning under $0.265\times$ energy consumption of the typical method.
\end{itemize}

\section{Methods}
In this section, \replaced{a}{the} theory for ANN-SNN conversion is first introduced. Based on this, \deleted{the }Rate Norm Layer with trainable threshold is thus proposed. Then, we analyse the reason for slow inference and suggest optimization for fit of firing rates. Finally, we present \replaced{a}{the} stage-wise learning strategy for accurate and fast SNN.

\subsection{Theory for Conversion from ANN to SNN}
The fundamental principle of ANN-SNN conversion is to match\replaced{ analog neurons' activation}{ an analog neuron's activation} with \replaced{spiking neurons' firing rate}{the firing rate of a spiking neuron}. One common way is to convert ReLU nonlinearity activation to the Integrate-and-Fire (I\&F) neuron. To be specific, for analog neurons in layer $l$ ($l=1,2,...,L$),
the ReLU activation can be described by:
\begin{align}
    \label{ann}
	\bm{a}_l = \max(W_{l-1}\bm{a}_{l-1}+\bm{b}_{l-1},0),
\end{align}
where vector $\bm{a}_l$ is the output of all ReLU-based artificial neurons in layer $l$,
$W_{l-1}$ and $\bm{b}_{l-1}$ is the weight and the bias term for the neurons in layer $l-1$.

As for the I\&F neuron, 
the membrane potential $v_{l}^i(t)$ for the $i$-th neuron in layer $l$ is formulated by:
\begin{align}
\label{eq:dynamic}
    \frac{dv^i_{l}(t)}{dt} = \sum_j \sum_{t_j\in T_j} W^{ij}_{l-1} \delta (t-t_j)+{b}_{l-1}^i,
\end{align}
where $b_{l-1}^i$ denotes the \replaced{bias}{input} current to the $i$-th neuron, $W^{ij}_{l-1}$ denotes the synaptic weight between the $j$-th presynaptic neuron in layer $l-1$ and the $i$-th neuron in layer $l$. $\delta (\cdot)$ is the delta function. $T_j$ denotes the set of \deleted{the }spike time\deleted{s} of the $j$-th presynaptic neuron, i.e., $T_j=\{t_j^{(1)},t_j^{(2)},\dots,t_j^{(K)}\}$.
When the membrane potential $v_{l}^i(t)$ exceeds the firing threshold $v_{th,l}$ in layer $l$, a spike is generated and the membrane potential $v_{l}^i(t)$ is reset to the rest value $v_{rest}<v_{th,l}$. 

To match \replaced{analog neurons' activation}{an analog neuron’s activation} with \replaced{spiking neurons'}{a spiking neuron's} firing rate, we discretize and vectorize Eq.~\ref{eq:dynamic} into time-steps and obtain the \replaced{spiking}{reset-by-subtraction} neuron\added{ model} for layer $l$.
\begin{align}
	\label{eq:neuron}
	   \bm{m}_l(t) &=  \bm{v}_l(t-1) + W_{l-1} \bm{s}_{l-1}(t)+\bm{b}_{l-1}, \nonumber \\
		\bm{s}_l(t) &= U(\bm{m}_l(t)-{v}_{th,l}),\\
		\bm{v}_l(t) &= \bm{m}_l(t)-{v}_{th,l}\bm{s}_l(t), \nonumber
\end{align}
where $\bm{m}_l(t)$ and $\bm{v}_l(t)$ represent the membrane potential of all I\&F neurons in layer $l$ after neuronal dynamics and after the trigger of a spike at time $t$, $U(\cdot)$ is the Heaviside Step Function, $\bm{s}_l(t)$ denotes the vector of \deleted{the }binary spike\added{s}, the element \replaced{in}{of} which equals 1 if there is a spike and 0 otherwise.
$\bm{b}_{l-1}$ is the vector of ${b}_{l-1}^i$, and $W_{l-1}$ is the weight matrix.
Note that here we use the "soft reset" \cite{han2020rmp} instead of the "hard reset".
At the moment of a spike, the membrane potential $\bm{v}_l(t)$ \replaced{is reduced}{reduces} by an amount equal to the firing threshold $v_{th,l}$, instead of going back to the reset value.  

Based on these \replaced{definitions}{definition}, we can derive the relationship between the firing rate $\bm{r}_l(t)$ of spiking neurons in layer $l$ and $\bm{r}_{l-1}(t)$ of neurons in layer $l-1$, which is depicted in Lemma~\ref{theorem:firing_rate}.  The proof can be found in the Appendix.

\begin{lemma} 
\label{theorem:firing_rate}
For a spiking neural network \replaced{consisting}{consists} of the reset-by-subtraction neurons mentioned in Eq.~\ref{eq:neuron}, assume that $W_{l-1}$ and $\bm{b}_{l-1}$ are the parameters for layer $l-1$. Then when $t \rightarrow \infty$, the relation of the firing rate $\bm{r}_l(t)$ and $\bm{r}_{l-1}(t)$ is given by:
\begin{align}
	\bm{r}_l = {\rm clip} \left (\frac{W_{l-1}\bm{r}_{l-1}+\bm{b}_{l-1}}{v_{th,l}},0,1 \right ),
	\label{eq:accu_rate}
\end{align}
where ${\rm clip}(x,0,1)=x$ when $x\in [0,1]$, ${\rm clip}(x,0,1)=1$ when $x > 1$, and ${\rm clip}(x,0,1)=0$ when $x < 0$. 
\end{lemma}
 With Lemma 1 and Eq.~\ref{ann}, We can derive the theorem for conversion from ANN to SNN:

\begin{theorem}
\label{cor:ann_snn}
For an $L$-layered ANN with the ReLU activation and an $L$-layered SNN with the reset-by-subtraction neurons, assume that $W^{ \text{ANN}}_{l-1},\bm{b}^{\text{ANN}}_{l-1}$ are the parameters for layer $l-1$ of \replaced{the}{an} ANN, 
and $W^{\text{SNN}}_{l-1},\bm{b}^{\text{SNN}}_{l-1}$ are the parameters for layer $l-1$ of \replaced{the}{an} SNN. $\max_{l}$ is the maximum activation of layer $l$ in ANN, and $v_{th,l}$ is the firing threshold of layer $l$ in SNN. The ANN can be converted to the SNN \added{when $t \rightarrow \infty$} (Eq.~\ref{ann} equals Eq.~\ref{eq:accu_rate})
if for $l=1,2,...,L$, the following equations hold:
\begin{align}
\label{eq:equivalent_condition}
    \frac{W^{\text{SNN}}_{l-1}}{v_{th,l}} = W^{\text{ANN}}_{l-1} \frac{\max_{l-1}}{\max_{l}},~~~
        \frac{\bm{b}^{\text{SNN}}_{l-1}}{v_{th,l}} = \frac{\bm{b}^{\text{ANN}}_{l-1} }{\max_{l}}.
\end{align}
\end{theorem}
The proof of Theorem~\ref{cor:ann_snn} is presented in the Appendix. Eq.~\ref{eq:equivalent_condition} implies that scaling operations are necessary to convert ANN to SNN, either scaling weights (i.e. weight normalization) or setting thresholds (i.e. threshold balancing). In this sense, weight normalization is equivalent to threshold balancing.

\subsection{Rate Norm Layer}
\label{sec:rnl}
The choices of scaling factors are often empirical, and post-training \cite{rueckauer2016theory,han2020rmp}. To overcome this, we propose Rate Norm Layer (RNL) to replace the ReLU activation in ANN. The idea is to use a clip function with a trainable upper bound to output the simulated firing rate, which is the limitation of actual firing rate\deleted{s} in SNN \replaced{when inference time $T \rightarrow \infty$}{as time step increases}. \added{Here we denote the simulated firing rate as $\hat{\bm{r}}_l$ ,} Rate Norm Layer can be formally expressed as follows:
\begin{align}
\label{rnl}
    {\theta}_l &= p_l \cdot \max(W_{l-1} \replaced{\hat{\bm{r}}_{l-1}}{\bm{z}_{l-1}}+\bm{b}_{l-1}),  \nonumber\\
	\bm{z}_l &={\rm clip}(W_{l-1} \replaced{\hat{\bm{r}}_{l-1}}{\bm{z}_{l-1}}+\bm{b}_{l-1},0,{\theta}_l), \\
	\hat{\bm{r}}_l &= \frac{\bm{z}_l}{\theta_l} \deleted{\quad \{ \textbf{simulated~firing~rate~in~ANN} \}}, \nonumber
\end{align}
where \replaced{$p_l$ is a trainable \replaced{scalar}{parameter} $(p_l \in [0,1])$, and $\theta_l$ is the threshold of the $l$-th layer}{$\theta_l$ is a dynamic threshold of the $l$-th layer. $p_l$ is a trainable parameter and is restricted in $[0,1]$}. \replaced{With Theorem~\ref{cor:ann_snn} satisfied ($v_{th,l}=\theta_l$) and $p_l=1$, one can find that Eq.~\ref{rnl} is equivalent to Eq.~\ref{eq:accu_rate}}{Comparing Eq.~\ref{eq:accu_rate} and Eq.~\ref{rnl}, one can find that an ANN with RNL can be directly converted to an SNN without any scaling}. \replaced{In this case}{When $p_l=1$}, RNL will degenerate to the Max Norm algorithm, which scales the weight $W_{l-1}$ by $\frac{\max_{l-1}}{\max_{l}}$ and the bias $\bm{b}_{l-1}$ by $\frac{1}{\max_{l}}$. \added{A diagram comparison of different scaling schemes can be seen in Fig.~\ref{fig:rnl_response}.} \deleted{\replaced{Analyses in Section \ref{subsec:opt_time}}{The following analysis} will show that if the optimization \replaced{for fast inference}{of inference time} is deployed, the upper limit of $p_l$ may be unnecessary. }\deleted{A \added{diagram }comparison of different scaling schemes can be seen in the appendix. }For mini-batch training, different batches have different maximum outputs. To reduce the perturbation caused by data sampling, \replaced{$\rm running\_max(W_{l-1} \hat{\bm{r}}_{l-1}+\bm{b}_{l-1})$}{${\rm running\_max}(\bm{z}_l)$} is used instead of \replaced{$\max(W_{l-1} \hat{\bm{r}}_{l-1}+\bm{b}_{l-1})$}{$\max(\bm{z}_l)$}. 
\begin{figure}[t]
	\begin{center}
		\includegraphics[width=0.7\linewidth,trim={10 10 10 10},clip]{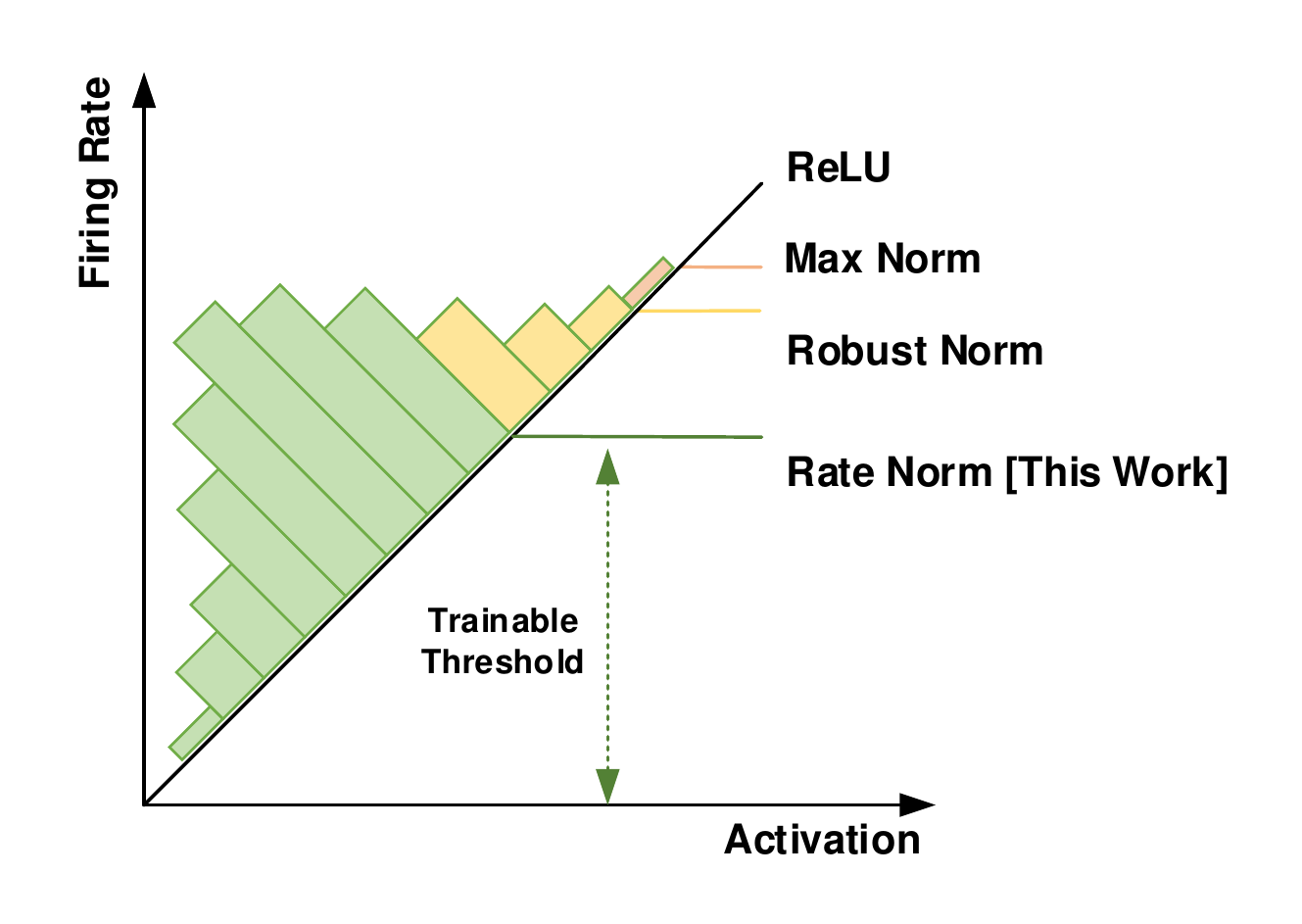}
	\end{center}
	\caption{Response of simulated firing rate of Rate Normalization Layer with regard to ANN activation. The trainable threshold can adapt during training, and manual setting can be cancelled.}
	\label{fig:rnl_response}
\end{figure}
The design of the Rate Norm Layer mainly considers the following three factors:
\begin{itemize}
	\item Compared with directly cutting the \added{simulated }firing rate to 1, the backpropagation becomes more effective in RNL training. \replaced{$\max(W_{l-1} \hat{\bm{r}}_{l-1}+\bm{b}_{l-1})$ and $\rm running\_max(W_{l-1} \hat{\bm{r}}_{l-1}+\bm{b}_{l-1})$}{$\max(\bm{z}_l)$ and ${\rm running\_max}(\bm{z}_l)$} can enable the gradient to flow out smoothly. Their participation in threshold calculating is similar to Batch Norm (BN) in mini-batch training. However, RNL cannot replace BN, because BN rescales the data to a normal distribution, which is related to the characteristics of ANN. \deleted{RNL more guarantees the similarity with ReLU at the conversion level}.
	
	\item The threshold $\theta_l$ enables better generalization. Existing solutions mainly focus on using a subset of the training data for offline normalization. This will potentially influence the generalization of the scaled SNN for data out of subset. In contrast, $\theta_l$ uses all data in training, which can be used directly in SNN inference.
	
	\item The threshold $\theta_l$ is determined in training. For faster inference, Robust Norm requires empirical percentile\deleted{,} which RNL doesn't need\added{,} as $p_l$ is trainable. Certainly, just using ANN loss will not guide the model to reduce the inference delay. This also requires additional loss design, which will be shown later.
\end{itemize}

\subsection{Optimization for Fast Inference}
\label{subsec:opt_time}

\begin{figure}[t]
	\begin{center}
		\includegraphics[width=0.8\linewidth,trim={25 0 25 40},clip]{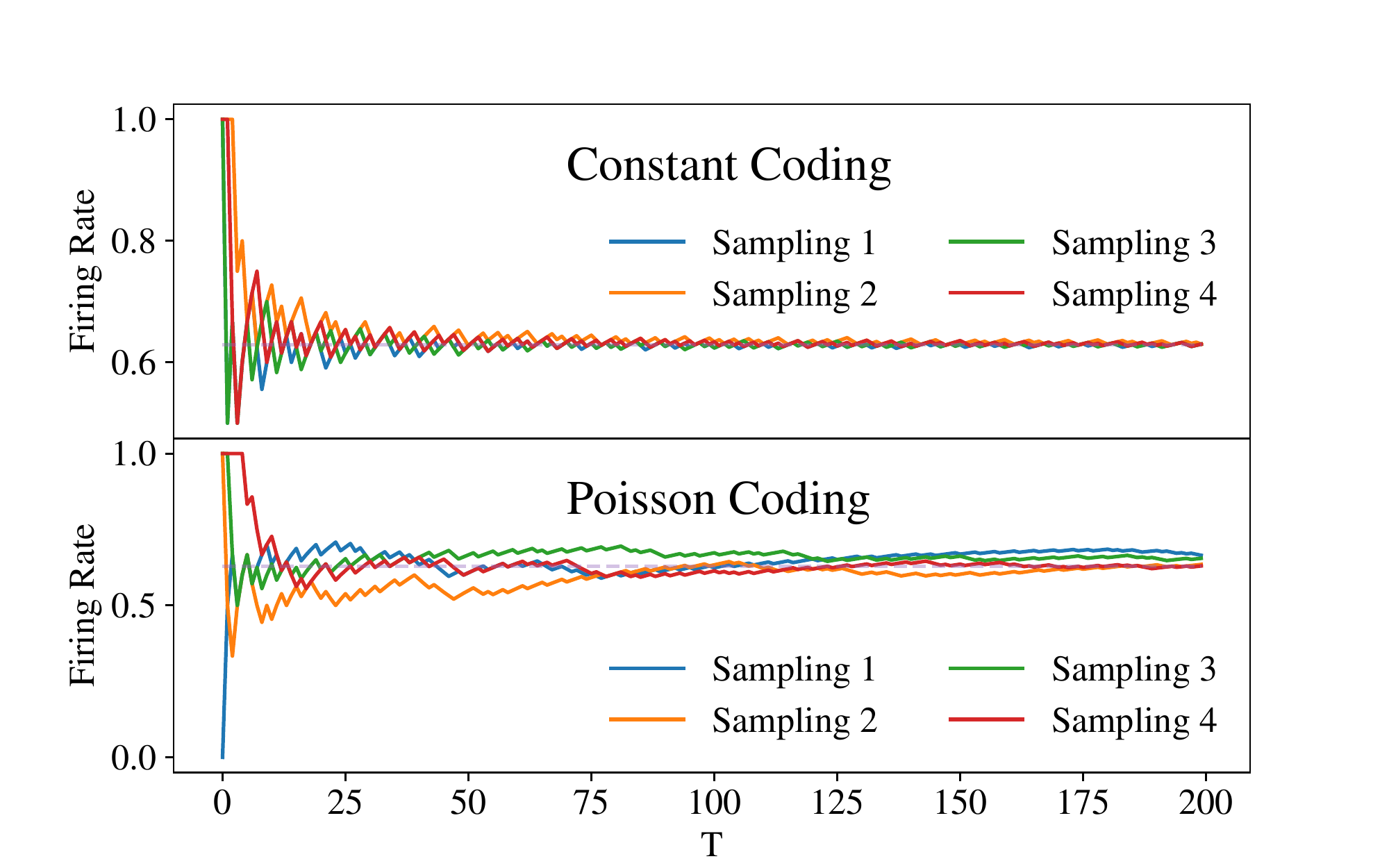}
	\end{center}
	\caption{Firing rate curves of output neurons in a toy SNN with different rate coding schemes. The membrane potential of spiking neurons has been randomly initialized from independent samplings.}
	\label{fig:rate_based_line}
\end{figure}

\replaced{Rate}{The rate}-based SNN model\added{s} take\deleted{s} rate coding as input, and the time average of the output spikes as output. In the conversion method\added{s}, the coding scheme mainly consists of two ways. One is Poisson coding, of which spikes obey the Poisson process. The other is constant coding, to cooperate with the reset-by-subtraction neuron. Constant coding is not a new thing, it can be regarded as an integrating ADC in the signal processing \cite{eng1994multiple}. 
The two primary rate coding forms take time for the firing rate to approach its expectations. For the constant coding, its accumulated current needs to be rounded down when converted to spike count\added{s}. So the firing rate will be jagged and approach the analog value (Fig.~\ref{fig:rate_based_line}). 
Using both codings will bring about unpredictable rate output in the first few time-steps. Following the suggestions of previous literature, constant coding is chosen as the primary scheme in this paper.

The time when the output firing rate of an SNN matches the analog output of an \replaced{ANN}{SNN} is referred to as ``inference time'', ``inference latency'' or ``inference delay'' \cite{neil2016learning}. Fig.~\ref{fig:rate_based_line} implies that for both rate coding schemes, there will be an output delay. In deep neural networks, the stacking of layer-by-layer delays will bring greater inference delay. For example, ResNet-44 requires 350 time-steps to achieve the best accuracy \cite{hu2018spiking}. This problem is not limited to ANN-SNN Conversion. The \replaced{BPTT}{directly} trained SNN model also has similar issues. 

Now that the reason for slow inference is attributed to the deviation of the encoder accumulation. To further analyze the characteristics, we propose to use $K(\hat{\bm{r}},\bm{r}(t))$ (K curve) to quantify the relationship between the simulated firing rate $\hat{\bm{r}}$ of ANN and the real firing rate $\bm{r}(t)$ of SNN after conversion. 
\begin{align}
	K(\hat{\bm{r}},\bm{r}(t))= \frac{{\Vert \bm{r}(t) - \hat{\bm{r}} \Vert}_2^2}{{\Vert \hat{\bm{r}} \Vert}_2^2}.
\end{align}
Note that the design of $K$ resembles the chi-square test\deleted{s} in hypothesis testing. $\hat{\bm{r}}$ and $\bm{r}(t)$ denote the firing rates of all \replaced{neurons}{the neuron} in a certain layer. $\Vert \cdot \Vert_2$ indicates \deleted{the }$L^2$ norm. The denominator ${{\Vert \hat{\bm{r}} \Vert}_2^2}$ makes $K(\hat{\bm{r}},\bm{r}(t))$ have scale invariance. Therefore, we can compare the fitting of the firing rate between different layers. Ideally, given enough ``inference time'', $K(\hat{\bm{r}},\bm{r}(t))$ will converge to 0. We believe K curve is an adequate metric as the neuron population encoding information is considered rather than any single neuron.

\begin{figure}[t]
    \scriptsize
    \begin{overpic}[width=1\linewidth,trim={90 -30 110 35},clip]{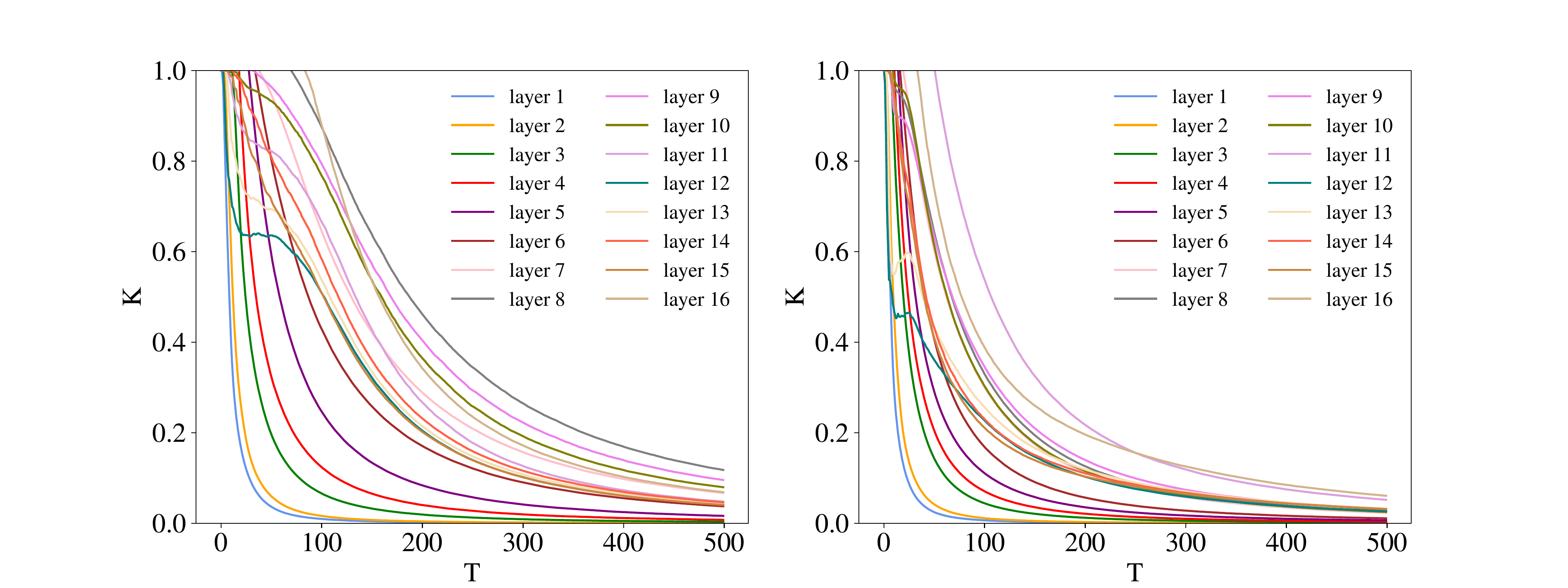}
		\put(25.5,1){(a)}
		\put(76.5,1){(b)}
	\end{overpic}
	\caption{Layer-wise K curves over time of \added{a }VGG16 model. (a) The result of SNN using the Max Norm algorithm; (b) The result of scaling the threshold\added{s} by 0.8.}
	\label{fig:opt_inference}
\end{figure}

Specifically, Fig.~\ref{fig:opt_inference} gives an example to illustrate how the K curve fits between different layers of \added{a }VGG16 and \added{a }converted SNN. An image is used for reasoning and calculating K curves. As the layer deepens, the convergence speed of the K curve becomes slower. \deleted{The time $t_s$ at which the deep curve begins to fall continues to be delayed. It indicates that not until $t_s$ do the spikes generate. }By accelerating the convergence of the K curve, the inference can speed up. Here we derive one of the upper bound for $K(\hat{\bm{r}},\bm{r}(t))$.

\begin{theorem}
\label{theorem:upper_bound}
For layer $l$ in an ANN and the converted SNN with constant coding, given the simulated firing rate $\hat{\bm{r}}_l$ and the real firing rate $\bm{r}_l$, we have:
\begin{equation}
    \label{eq:upper_bound}
    K_l < \frac{2\Omega_l}{t}
    ,
\end{equation}
where $K_l$ denotes the abbreviation of $K(\hat{\bm{r}}_l,\bm{r}_l(t))$ in layer $l$, $\Omega_l=\frac{ \Vert \hat{\bm{r}}_l \Vert_1 }{\Vert \hat{\bm{r}}_l \Vert_2^2} $. $\Vert \cdot \Vert_p$ denotes $L^p$ norm.
\end{theorem}
The detailed proof of Theorem~\ref{theorem:upper_bound} is described in the Appendix.
$\Omega_l$ is named as Rate Inference Loss (RIL) of the $l$-th layer. Eq.~\ref{eq:upper_bound} indicates that $K_l$ is less than an inverse proportional curve to $t$.

For the convergence of the last layer (the $L$-th layer), if $\Omega_L$ reduces, $K_L$ will converge faster to 0 due to the Squeeze Theorem. That is, the real firing rate of the SNN approaches the simulated value of the ANN more faster, leading to faster and more stable outputs of SNN. However, considering that the network actually has inference delays layer by layer, a better solution is to reduce the average value of $\Omega_l  (l=1,2,\dots,L)$. Thus, the overall training loss is composed of the loss related to the task and RIL multiplied by hyperparameter $\lambda$.
\begin{align}
	\mathscr{L}'(f(\bm{x}),\bm{y})=\mathscr{L}(f(\bm{x}),\bm{y})+ \lambda \frac{\sum{\Omega_l}}{L},
	\label{eq:total_loss}
\end{align}
where $(\bm{x},\bm{y})$ is data tuple for training and $f(\cdot)$ is the network with L Rate Norm Layers. Based on Eq.~\ref{rnl}, we are able to calculate the partial derivative of $\Omega_L$ w.r.t.~$p_L$, and obtain:
\begin{align}
	\frac{\partial \Omega_L}{\partial p_L} &= \sum_i{\frac{\partial \Omega_L}{\partial \hat{\bm{r}}_{L,i}}  \frac{\partial \hat{\bm{r}}_{L,i}}{\partial p_L}  } \nonumber\\
	&= \sum_i{\bigg[ \bigg( \frac{  \Vert \hat{\bm{r}}_{L}\Vert_2^2  -2\hat{\bm{r}}_{L,i} \Vert  \hat{\bm{r}}_{L,i} \Vert_1    }{\Vert \hat{\bm{r}}_{L} \Vert_2^4} \bigg) \bigg( - \frac{\hat{\bm{r}}_{L,i}}{p_L} \bigg) \bigg] } \nonumber\\
	&= \frac{\Vert \hat{\bm{r}}_L \Vert_1 }{p_L \Vert \hat{\bm{r}}_L \Vert_2^2  },
	\label{eq:partial}
\end{align}
where $\hat{\bm{r}}_{L,i}$ denotes the $i$-th element of $\hat{\bm{r}}_{L}$. 

Eq.~\ref{eq:partial} implies that the partial derivative of $\Omega_L$ w.r.t. $p_L$ is positive. Simply minimizing $\Omega_L$ will reduce the neurons’ $p_l$. The upper limit of $p_l$ is unnecessary in this sense. Nevertheless, this will lead more neurons to saturation state and \replaced{lose}{loss} model accuracy. \replaced{Thus, we jointly optimize the two losses and tune the hyper-parameter $\lambda$ to reach the optimal trade-off between model accuracy and conversion loss.}{Thus, joint learning of task loss and rate inference loss will greatly alleviate the problem.}

So far, the current theories and analyses can also systematically explain the findings of Han \etal That is, reasonable scaling of the threshold can speed up the inference. If the threshold is set smaller, the Rate Inference Loss will decrease. Since the curve \added{value }and time are inversely proportional, it is equivalent to accelerating the decrease of the curve. On the contrary, if the threshold is too large, many neurons in the SNN will have \deleted{a }low firing rate\added{s} to produce accurate output. Of course, the threshold should not be too small to prevent more neurons from saturating and losing information. Fig.~\ref{fig:opt_inference}(b) shows the K curve after adjusting the threshold using the algorithm of Han \etal

\subsection{Training for Accurate and Fast SNN}

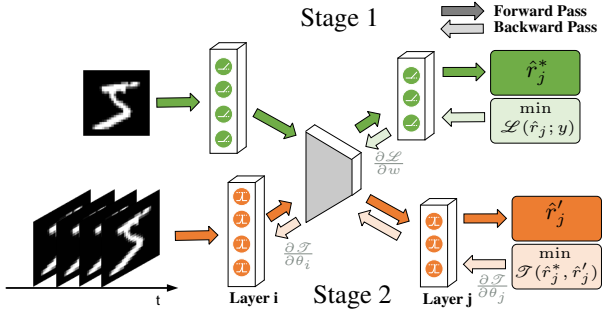
\begin{figure}[t]
	\begin{center}
		\begin{overpic}[width=1\linewidth]{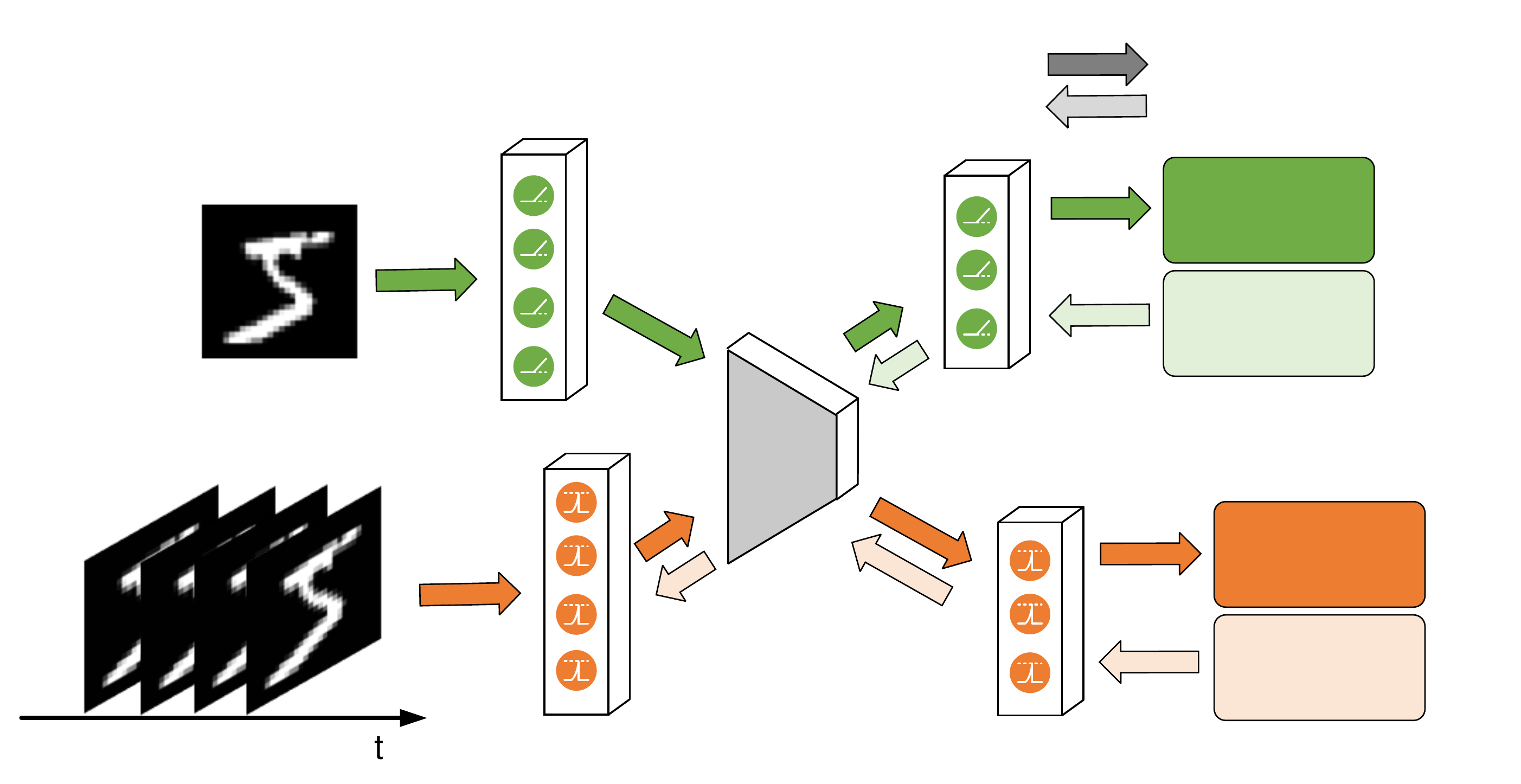}
			\put(47,45){\replaced{Stage 1}{ANN}}
			\put(49,2){\replaced{Stage 2}{SNN}}
			
			\put(36,2){\tiny \textbf{Layer i}}
			\put(66,1.8){\tiny \textbf{Layer j}}
			\put(77,46.5){\tiny \textbf{Forward Pass}}
			\put(77,44){\tiny \textbf{Backward Pass}}
			\put(82,37){\small $\hat{r}_j^*$}
			\put(81,31.5){\tiny $\min$}
			\put(78,28.5){\tiny $\mathscr{L}(\hat{r}_j;y)$}
			
			\put(85,15){\small $\hat{r}_j'$}
			\put(84,9){\tiny $\min$}
			\put(80,6){\tiny $\mathscr{T}(\hat{r}_j^*,\hat{r}_j')$}
			
			\put(44,9){\tiny \textcolor[rgb]{0.5,0.54,0.52}{$\frac{\partial \mathscr{T}}{\partial \theta_i}$}}
			\put(74,4){\tiny \textcolor[rgb]{0.5,0.54,0.52}{$\frac{\partial \mathscr{T}}{\partial \theta_j}$}}
			\put(58,23){\tiny \textcolor[rgb]{0.5,0.54,0.52}{$\frac{\partial \mathscr{L}}{\partial w}$}}
		\end{overpic}
	\end{center}
	\caption{Diagram of SNN training. The training consists of two stages. It first \replaced{trains for accuracy}{minimizes precision loss} by adjusting weights. Then it optimizes for fast inference by adjusting firing thresholds.}
	\label{fig:net_flow}
\end{figure}

In Section~\ref{sec:rnl}, the $p_l$ of the Rate Norm Layer is restricted to $[0,1]$. This means if the Rate Norm Layer is directly trained, the simulated firing rate after clipping and scaling will inevitably bear information loss. The information loss here indicates that the output distribution due to cropping is different from the original distribution. This will make it difficult to take advantage of ANN's performance. So the training strategy needs to be carefully designed. When $p_l=1$ is fixed, $\bm{z}_l$ is clipped with ${\rm running\_max}(\bm{z}_l)$, which has less information loss. \replaced{After}{When} the network\added{ synaptic} parameters are fixed, the neuron threshold starts training. The training goal at this stage is to reduce information loss and to reduce Rate Inference Loss. As the output  $f(\bm{x})$ of an ANN is $\hat{\bm{r}}_L$, the goal is to optimize (according to Eq.~\ref{eq:total_loss}):
\begin{align}
	\mathscr{L}'(\hat{\bm{r}}_L,\bm{y})=\mathscr{L}(\hat{\bm{r}}_L,\bm{y})+\lambda\frac{\sum{\Omega_L}}{L},
\end{align}
we decompose this goal into two stage goals (shown in Fig.~\ref{fig:net_flow}):

Stage 1 is accuracy training, when network outputs $\hat{\bm{r}}^*_L$. The target is:
\begin{align}
	\min_{W,b} \mathscr{L}(\hat{\bm{r}}^*_L,\bm{y} ).
\end{align}

Stage 2 is for fast inference. Assume the output of this stage is $\hat{\bm{r}}'_L$. Then the target is: 
\begin{align}
	&\min_{\theta} \mathscr{T}(\hat{\bm{r}}_j^*,\hat{\bm{r}}_j'), \nonumber\\
    \mathscr{T}(\hat{\bm{r}}_j^*,\hat{\bm{r}}_j') =& 1- Cos(\hat{\bm{r}}^*_L,\hat{\bm{r}}'_L) + \lambda \frac{\sum{\Omega_L}}{L}.
    \label{eq:fast_loss}
\end{align}
The cos distance is used to maintain neuron information. \replaced{The detailed training and converting algorithm is described in Algorithm~\ref{algo:training} in the Appendix}{The detailed training algorithm is described in the appendix}.

\section{Experiments}

\subsection{Experiment Implementation}
We validate our methods on the image recognition benchmarks, namely the MNIST\footnote{http://yann.lecun.com/exdb/mnist/}, CIFAR-10, CIFAR-100\footnote{https://www.cs.toronto.edu/~kriz/cifar.html} datasets. For MNIST, we consider a 7-layered CNN and AlexNet. For CIFAR-10, we use VGG-16 and PreActResNet-18 network structure\added{s}. It is worth noting that we did not use the common ResNet as we think PreActResNet will help the training of Rate Norm Layer\added{s} \cite{he2016identity}. For CIFAR-100, VGG-16, PreActResNet-18 and PreActResNet-34 are used. 
\deleted{The choice of all hyperparameters is shown in Table S1 in the Appendix. }
We present the simulation results and analysis in the following subsections.

\begin{table}
\centering
\scriptsize
\setlength\tabcolsep{5pt}
\begin{tabular}{lccc}
\toprule
\quad  & Network & SNN & Conversion\\
\quad  & \quad & Acc (\%) & Loss (\%)\\
\hline
\multicolumn{4}{c}{MNIST} \\
\hline
\cite{diehl2015fast} & Spiking NN & 98.6 & - \\
\cite{rueckauer2017conversion} & - & 99.44 & 0.00  \\
\textbf{This work (RNL)} & 7-Layered CNN & 96.51 & 0.00 \\
\textbf{This work (RNL)} & AlexNet & 99.46 & -0.04 \\
\hline
\multicolumn{4}{c}{CIFAR-10} \\
\hline
\cite{sengupta2019going} & ResNet-20 & 87.46 & +1.64 \\
\cite{sengupta2019going} & VGG-16 & 91.55 & +0.15 \\
\cite{hunsberger2015spiking} & - & 83.54 & +0.18 \\
\cite{cao2015spiking} & 7-Layered CNN* & 77.43 & +1.69 \\
\cite{han2020rmp} & ResNet-20 & 91.36 & +0.11 \\
\cite{han2020rmp} & VGG-16 & 93.63 & 0.00 \\
\cite{rueckauer2017conversion} & - & 88.82 & +0.05 \\
\textbf{This work (RNL)} & VGG-16 & 92.86 & -0.04 \\
\textbf{This work (RNL)} & PreActResNet-18 & 93.45 & -0.39  \\
\hline
\multicolumn{4}{c}{CIFAR-100} \\
\hline
\cite{han2020rmp} & VGG-16 & 70.93 & +0.29  \\
\cite{han2020rmp} & ResNet-20 & 67.82 & +0.90 \\
\textbf{This work (RNL)} & VGG-16 & 75.02 & +0.54 \\
\textbf{This work (RNL)} & PreActResNet-18 & 75.10 & -0.45  \\
\textbf{This work (RNL)} & PreActResNet-34 & 72.91 & +0.80  \\
\bottomrule
\end{tabular}
\caption{Best Accuracy Performance comparing with related methods. Values in the table represent the best accuracy and the accuracy loss of conversion ($\text{Acc}_{\text{ANN}} - \text{Acc}_{\text{SNN}}$). The structure of 7-Layered CNN is 32C3-P2-32C3-P2-32C3-P2-32FC10, which is different from the one with asterisk in the table.}
\label{tab:best_infer}
\end{table}

\subsection{Accuracy Performance}

We first evaluate the effectiveness of the proposed Rate Norm Layer. \replaced{A network}{Network} with Rate Norm Layer\added{s} is trainable with backpropagation. When testing SNN performance, all Rate Norm Layers are converted to I\&F neuron models with $\theta_l$ as the threshold. Table~\ref{tab:best_infer} shows the best accuracy of the converted SNNs compared with typical works.
The converted SNNs achieve \added{the }state-of-the-art performance on MNIST and CIFAR-100 datasets, and reach a similar performance\deleted{ level} on CIFAR-10. For VGG-16 trained by CIFAR-100, the proposed method reaches top-1 accuracy 75.02\%, whereas the state-of-the-art ANN-SNN algorithm reaches 70.93\%.
Conversion loss is considered here to evaluate the quality of the ANN-SNN conversion. As illustrated in Table~\ref{tab:best_infer},
ANN-SNN conversion with \deleted{the }Rate Norm Layer\added{s} has low conversion loss, or even negative conversion loss, indicating that the converted SNN \replaced{may outperform}{outperforms} the original ANN. In contrast, the loss is usually positive for other methods, meaning that the performance of the converted SNN is not as good as ANN.

\subsection{Fast Inference Performance}

\begin{figure}[t]
\scriptsize
	\begin{center}
	    \includegraphics[width=0.9\linewidth,trim={30 0 50 45},clip]{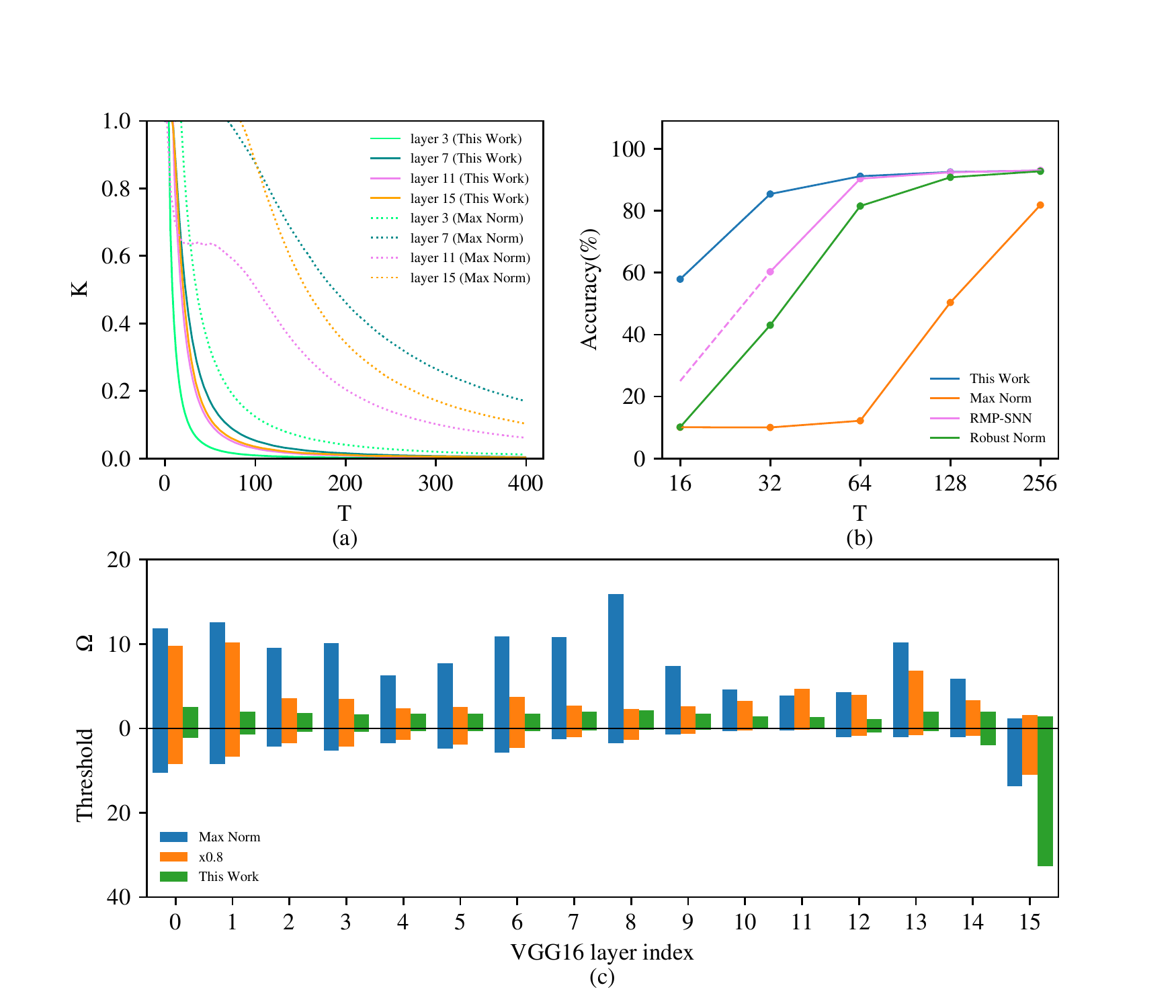}
	\end{center}
	\caption{Fast inference performance of SNNs converted from VGG-16 trained on CIFAR-10. (a) The K curve over time. For readability, 4 of the 16 layers of VGG-16 are extracted to display the K curve\added{s}. (b) Accuracy curves over time. The dotted line only indicates the tendency due to the lack of data. (c) $\Omega$ and threshold distribution.}
	\label{fig:fast_infer}
\end{figure}

We test whether the proposed Rate Inference Loss can speed up inference, that is, speed up the convergence of the K curve. 
Fig.~\ref{fig:fast_infer}(a) and (b) show how the K curve and accuracy change over latency, where the dotted line \added{in Fig.~\ref{fig:fast_infer}(a) }represents the Max Norm method. As can be seen from Fig.~\ref{fig:fast_infer}(a),  the K curves of the proposed method converge to 0 quickly, which are much faster than those of the Max Norm method. Thus the proposed method can implement fast inference.  
The inference performance can be observed in Fig.~\ref{fig:fast_infer}(b). The proposed method reaches an accuracy of 85.40\% using 32 time-steps, whereas the methods of Max Norm, Robust Norm, and RMP-SNN reach 10.00\%, 43.03\% and 63.30\% at the end of 32 time-steps. Moreover, the proposed method achieves an accuracy above 90\% using only 52 time-steps, which is 8.6 times faster than Max Norm that uses 446 time-steps. Detailed accuracy comparison on time T is shown in Table~\ref{tab:fast_infer}. 

The threshold and $\Omega$ of VGG-16 are visualized in Fig.~\ref{fig:fast_infer}(c).  For the Max Norm method, as all the I\&F neurons use the same threshold 1,  we regard the maximum value of ReLU outputs as the equivalent threshold. It can be found that the distribution gap of the threshold is relatively large. But when paying attention to the $\Omega$ distribution, the $\Omega$ after the threshold scaling (orange column) is usually smaller than that of Max Norm (blue column). In contrast, training with Rate Inference Loss will keep $\Omega$ at a relatively low and average level, and thus benefit inference.

\begin{table}
\centering
\scriptsize
\setlength\tabcolsep{3.25pt}
\begin{tabular}{lllllll}
\toprule
\quad  & Network & 16 & 32 & 64 & 128 & 256 \\
\midrule
Max Norm & VGG-16 & 10.07 & 10.00 & 12.17 & 50.37 & 81.85 \\
\cite{rueckauer2017conversion} & PreActResNet-18 & 11.75 & 11.75 & 21.08 & 55.25 & 78.33 \\

Robust Norm & VGG-16 & 10.11 & 43.03 & 81.52 & 90.80 & 92.75 \\
\cite{rueckauer2017conversion} & PreActResNet-18 & 13.50 & 13.00 & 23.50 & 59.50 & 80.50 \\

\cite{han2020rmp} & VGG-16 & - & 60.30 & 90.35 & 92.41 & 93.04\\

\textbf{This work (RNL+RIL)} & VGG-16 & 57.90 & 85.40 & 91.15 & 92.51 & 92.95 \\
\textbf{This work (RNL+RIL)} & PreActResNet-18 & 47.63 & 83.95 & 91.96 & 93.27 & 93.41 \\
\bottomrule
\end{tabular}
\caption{Fast Inference Performance comparing with related methods. Values in the table represent the instant accuracy of latency T. All the networks are trained on CIFAR-10.}
\label{tab:fast_infer}
\end{table}

\subsection{Energy Estimation of Neuromorphic Hardware}

SNNs have a considerable potential on neuromorphic chips. One of the benefits is to reduce the energy budget. To study the energy efficiency of fast reasoning, we use the energy model proposed by Cao et al. \cite{cao2015spiking} to model the energy consumption on neuromorphic chips. Assume\deleted{d} that a spike activity would bring about the energy consumption of $\alpha$ Joules and 1 time-step takes $1ms$. Then the power model is defined as:
\begin{align}
	P = \frac{total \, spikes}{1\times 10^{-3}} \times \alpha \, (Watts)
\end{align}

The previous energy analysis mainly focused on total energy consumption during inference. However, in real application scenarios, the total energy consumed before the model reaches reliable accuracy is more important. In this regard, we evaluate the performance of the proposed method and the Max Norm method on \replaced{energy}{power}. Fig.~\ref{fig:power_analysis}(a) is the power histogram over time. Due to the lower threshold, the power of our model is relatively high. The integral of the deep color area represents the energy consumed to achieve 90\% accuracy. The energy consumption of the proposed model is only 0.265 times of the Max Norm method when it reaches 90\% accuracy. This means that the 8.6 $\times$ reasoning speedup will not bring much more energy consumption. Fig.~\ref{fig:power_analysis}(b) shows the logarithmic ratio of energy and inference speedup. Our model exhibits \added{the properties of }``fast reasoning'' and ``energy efficiency''.

\begin{figure}[t]
\scriptsize
	\centering
	\begin{overpic}[width=1.0\linewidth,trim={110 5 110 50},clip]{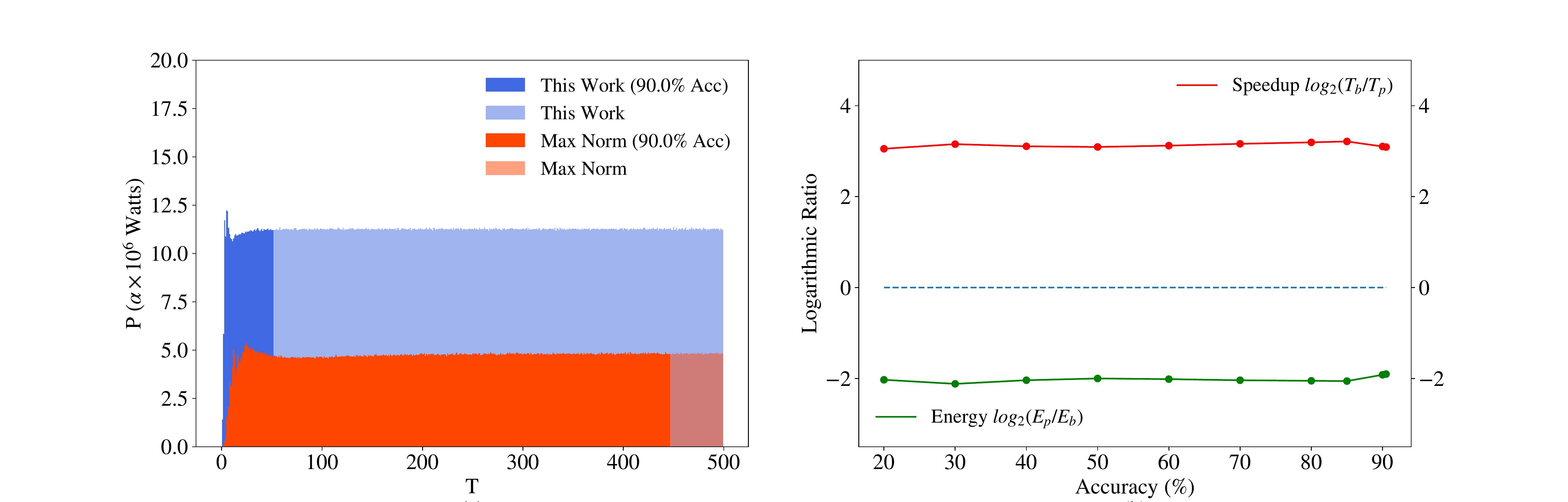}
		\put(25,-2.5){(a)}
		\put(76.5,-2.5){(b)}
	\end{overpic}
	\quad
	\caption{Energy analysis of VGG-16 on neuromorphic chips. (a) is the power histogram. The integral of the deep color area is the energy for accuracy to reach 90\%. (b) is the logarithmic ratio of energy consumption and inference speedup. `p' is for the proposed method, and `b' is for the baseline Max Norm method. }
	\label{fig:power_analysis}
\end{figure}

\section{Conclusions}

This paper proposes a method to \deleted{directly }convert conventional ANNs to SNNs. The Rate Norm Layer is introduced to replace ReLU for optimal conversion. Besides, we quantify the fit between \added{the }ANN activation and the firing rate of the converted SNN by an optimal fit curve. The inference time can be reduced by optimizing the coefficient of the upper bound of the fit curve, namely Rate Inference Loss. Thus, a \replaced{two-staged}{joint} learning scheme is proposed to obtain fast and accurate deep SNNs. Experimental results demonstrate that our methods achieve low accuracy loss and fast reasoning with deep structures such as VGG and PreActResNet.

\section*{Acknowledgment}
	This work was supported by the National Natural Science Foundation of China (62027804, 61825101, 62088102 and 61961130392).

\section*{Appendix A: Proofs}
\normalsize
\setcounter{equation}{0}
\setcounter{theorem}{0}
\setcounter{corollary}{0}
\setcounter{lemma}{0}

Here we provide the theoretic proofs of theorems and lemmas in this paper.
\begin{lemma} 
For a spiking neural network consist\replaced{ing}{s} of the reset-by-subtraction neurons mentioned in Eq.~3, assume that $W_{l-1}$ and $\bm{b}_{l-1}$ are the parameters for layer $l-1$. Then when $t \rightarrow \infty$, the relation of the firing rate $\bm{r}_l(t)$ and $\bm{r}_{l-1}(t)$ is given by:
\begin{align}
	\bm{r}_l = {\rm clip} \left (\frac{W_{l-1}\bm{r}_{l-1}+\bm{b}_{l-1}}{v_{th,l}},0,1 \right ),
\end{align}
where ${\rm clip}(x,0,1)=x$ when $x\in [0,1]$, ${\rm clip}(x,0,1)=1$ when $x>1$, and ${\rm clip}(x,0,1)=0$ when $x<0$. 
\end{lemma}

\begin{proof}
For the reset-by-subtraction spiking neurons formulated by Eq.~3 in the main text, we can stack the equations and get the discrete function between spikes of layer $l$ and layer $l-1$:
\begin{align}
	\bm{v}_l(t)-\bm{v}_l(t-1)=W_{l-1}\bm{s}_{l-1}(t)+\bm{b}_{l-1}-v_{th,l}\bm{s}_l(t).
\end{align}
By summing the left and right expression\added{s} over time and \replaced{dividing $t v_{th,l}$ on the both sides}{dividing by $t v_{th,l}$}, the equation can be reformulated as:
\begin{align}
        \bm{r}_l(t) = \frac{\sum( \bm{s}_l(t))}{t} = \frac{W_{l-1} \bm{r}_{l-1}(t)+\bm{b}_{l-1}}{v_{th,l}} - \frac{\bm{v}_l(t)}{t v_{th,l}}, 
\end{align}
where $\bm{r}_l (t)$ denotes the firing rates of all neurons in layer $l$. 

From the membrane potential updating function (Eq.~3), $\bm{v}_l (t)$ is in the range of $[0, v_{th,l}]$, thus we have: 
\begin{equation}
   \lim_{t\rightarrow \infty} \frac{\bm{v}_l (t)}{tv_{th,l}}= 0.
\end{equation}
As the value of $\bm{s}_l(t)$ can only be 0 or 1, the firing rate $\bm{r}_l (t)$ is strictly restricted in $[0,1]$. When $t \rightarrow \infty$, it's straightforward to conclude that
\begin{equation}
	\bm{r}_l = {\rm clip} \left (\frac{W_{l-1}\bm{r}_{l-1}+\bm{b}_{l-1}}{v_{th,l}},0,1 \right ).
\end{equation}
\end{proof}

\begin{theorem}
For an $L$-layered ANN with the ReLU activation and an $L$-layered SNN with the reset-by-subtraction neurons, assume that $W^{ \text{ANN}}_{l-1},\bm{b}^{\text{ANN}}_{l-1}$ are the parameters for layer $l-1$ of \replaced{the}{an} ANN, 
and $W^{\text{SNN}}_{l-1},\bm{b}^{\text{SNN}}_{l-1}$ are the parameters for layer $l-1$ of \replaced{the}{an} SNN. $\max_{l}$ is the maximum activation of layer $l$ in ANN, and $v_{th,l}$ is the firing threshold of layer $l$ in SNN \added{when $t \rightarrow \infty$}. The ANN can be converted to the SNN (Eq.~1 equals Eq.~4)
if for $l=1,2,...,L$, the following equations hold:
\begin{align}
    \frac{W^{\text{SNN}}_{l-1}}{v_{th,l}} = W^{\text{ANN}}_{l-1} \frac{\max_{l-1}}{\max_{l}},~~~
        \frac{\bm{b}^{\text{SNN}}_{l-1}}{v_{th,l}} = \frac{\bm{b}^{\text{ANN}}_{l-1} }{\max_{l}}.
\end{align}
\end{theorem}

\begin{proof}
The conversion of ANN to SNN is built on the basis of the equivalent of SNN firing rate and ANN activation. 
Considering spiking neurons of layer $l$ and $l-1$ in an SNN, the relationship between the firing rate $\bm{r}_l(t)$ and $\bm{r}_{l-1}(t)$ is:
\begin{equation}
	\bm{r}_l = {\rm clip} \left (\frac{W^{\text{SNN}}_{l-1} \bm{r}_{l-1}+\bm{b}^{\text{SNN}}_{l-1}}{v_{th,l}},0,1 \right ).
	\label{eq:cor1_1}
\end{equation}
Note that the firing rate $\bm{r}_{i} (i=1,2,\dots,L)$ in Eq.~\ref{eq:cor1_1} is restricted in $[0,1]$. But the ANN activation (ReLU) only satisfy $a_i \ge 0$ (Eq.~1 in the main text). In fact, for countable limited dataset, the activation generated by the network is also upper bounded. Assume that the upper bound for the output of all ReLU-based artificial neurons in layer $i$ is $\max_i$, we have:
\begin{align}
    0 \leq \bm{a}_i \leq {\rm max}_i.
\end{align}
Let $\bm{z}_i=\frac{\bm{a}_i}{\max_i}$, then $0 \leq \bm{z}_i \leq 1 \, (i=1,2,\dots,L)$. According to Eq.~1, the activation of layer $l$ and $l-1$ satisfy:
\begin{align}
    \bm{a}_l = \max \left( W^{\text{ANN}}_{l-1} \bm{a}_{l-1} +\bm{b}^{\text{ANN}}_{l-1} ,0 \right)
\end{align}
As the $\bm{a}_l$ clipped by $\max_l$ equals the original $\bm{a}_l$. Then, 
\begin{align}
    \bm{a}_l = {\rm clip} \left(W^{\text{ANN}}_{l-1} \bm{a}_{l-1} +\bm{b}^{\text{ANN}}_{l-1} ,0,{\rm max}_l \, \right)
    \label{eq:cor1_2}
\end{align}
By dividing $\max_l$ on both sides of Eq.~\ref{eq:cor1_2} and substituting $\bm{a}_i$ by $\bm{z}_i\max_i$, we have:
\begin{equation}
    \bm{z}_l = {\rm clip} \left( \frac{ W^{\text{ANN}}_{l-1} \bm{z}_{l-1}\max_{l-1} +\bm{b}^{\text{ANN}}_{l-1} }{\max_l}  ,0,1 \right)
    \label{eq:cor1_3}
\end{equation}
Comparing Eq.~\ref{eq:cor1_1} and Eq.~\ref{eq:cor1_3}, We can conclude that Eq.~\ref{eq:cor1_1} equals Eq.~\ref{eq:cor1_3} if for  $l=1,2,...,L$, the following equations hold:
\begin{align}
        \frac{W^{\text{SNN}}_{l-1}}{v_{th,l}} = W^{\text{ANN}}_{l-1} \frac{\max_{l-1}}{\max_{l}},~~\frac{\bm{b}^{\text{SNN}}_{l-1}}{v_{th,l}} = \frac{\bm{b}^{\text{ANN}}_{l-1} }{\max_{l}}
\end{align}

\end{proof}

\begin{theorem}
For layer $l$ in an ANN and the converted SNN with constant coding, given the simulated firing rate $\hat{\bm{r}}_l$ and the real firing rate $\bm{r}_l$, we have:
\begin{equation}
    K_l < \frac{2\Omega_l}{t},
\end{equation}
where $K_l$ denotes the abbreviation of $K(\hat{\bm{r}}_l,\bm{r}_l(t))$ in layer $l$, $\Omega_l=\frac{ \Vert \hat{\bm{r}}_l \Vert_1 }{\Vert \hat{\bm{r}}_l \Vert_2^2} $. $\Vert \cdot \Vert_p$ denotes $L^p$ norm.
\end{theorem}

\begin{proof}
Consider that firing is a cumulative and rounded firing process. When $\hat{\bm{r}}_l$ is given, ${\bm{r}}_l(t)$ is approximate as $\frac{\lfloor \hat{\bm{r}}_l t \rfloor}{t}$, also
$\hat{\bm{r}}_{l,i}-\frac{1}{t} < \frac{\lfloor \hat{\bm{r}}_{l,i} t \rfloor}{t} \le \hat{\bm{r}}_{l,i}$, where $\hat{\bm{r}}_{l,i}$ denotes the $i$-th element of the vector $\hat{\bm{r}}_{l}$. For the $l$-th layer, we have:
\begin{equation}
	\begin{split}
		\bigg \Vert \bm{r}_l (t)- \hat{\bm{r}}_l \bigg \Vert_2^2 &=\bigg \Vert \frac{\lfloor \hat{\bm{r}}_l t \rfloor}{t} - \hat{\bm{r}}_l \bigg \Vert_2^2 =
		\sum_i \left[ \frac{\lfloor \hat{\bm{r}}_{l,i} t \rfloor}{t} - \hat{\bm{r}}_{l,i} \right]^2 \\
		&= \sum_i \left[ \left( \frac{\lfloor \hat{\bm{r}}_{l,i} t \rfloor}{t} \right)^2 - 2 \hat{\bm{r}}_{l,i} \frac{\lfloor \hat{\bm{r}}_{l,i} t \rfloor}{t} + \hat{\bm{r}}^2_{l,i} \right]\\
		&< \sum_i \left[ \hat{\bm{r}}_{l,i}^2 - 2 \hat{\bm{r}}_{l,i} \left( \hat{\bm{r}}_{l,i}-\frac{1}{t} \right) + \hat{\bm{r}}^2_{l,i}\right]\\
		& = \sum_i \left[ \hat{\bm{r}}^2_{l,i} - 2\hat{\bm{r}}^2_{l,i} +2\hat{\bm{r}}_{l,i} \frac{1}{t} + \hat{\bm{r}}^2_{l,i} \right]\\
		& = \sum_i \frac{2 \hat{\bm{r}}_{l,i}}{t} = \frac{2 \Vert \hat{\bm{r}}_l \Vert_1}{t} 
	\end{split}
\end{equation}
The last equality holds as simulated firing rate $\hat{\bm{r}}_l \ge 0$. Thus the sum of all items in $\hat{\bm{r}}_l$ equals its $L^1$ norm. Now we conclude that:
\begin{equation}
	\begin{split}
		K_l &= \frac{\Vert \bm{r}_l(t) - \hat{\bm{r}}_l \Vert_2^2 }{\Vert \hat{\bm{r}}_l \Vert_2^2}  < \frac{2 \Vert \hat{\bm{r}}_{l} \Vert_1 }{\Vert \hat{\bm{r}}_l \Vert_2^2 t}  = \frac{2\Omega_l}{t} 
	\end{split}
\end{equation}
\end{proof}

\begin{algorithm}
	\caption{Mini-batch training and converting a spiking neural network from a source ANN with Rate Norm Layers. `r\_max' is short for `running\_max'.}
	\label{algo:training}
	\textbf{Require}:A spiking neural network $f_{SNN}$ with thresholds \{$v_{th,k}$|$k$=$1,2,\cdots,L$\}; A network with $L$ Rate Norm Layers $f_{ANN}$; Dataset $D$ for training; Number of epochs for Stage 1 $epoch1$; Number of epochs for Stage 2 $epoch2$ \\
	\textbf{Ensure}: $p_k$=$1.0$, ${\rm r\_max}_k$=$1.0$ $(k=1,2,\dots,L)$; momentum parameter $m=0.1$, $\lambda=0.5$
	\begin{algorithmic}[1]
	    \STATE \textbf{\{Stage 1: Accuracy Training for $f_{ANN}$\}}
	    \FOR{e = 1 to $epoch1$}
	    \FOR{length of Dataset $D$}
	    \STATE Sample minibatch $(\bm{x},\bm{y})$ from $D$
	    \STATE $\hat{\bm{r}}^*_0=\bm{x}$
	    \FOR{k = 1 to L}
	    \STATE ${\rm r\_max}_k$ = (1-$m$) ${\rm r\_max}_k$ +$m$  $\max(\hat{\bm{r}}^*_{k-1})$
	    \STATE $\theta_k$ = $p_k$  ${\rm r\_max}_k$
	    \STATE $\hat{\bm{r}}^*_{k}$ = clip[$(W_{k-1} \hat{\bm{r}}^*_k + \bm{b}_{k-1})/\theta_k$,0,1]
	    \ENDFOR
	    \STATE Loss = $\mathscr{L}(\hat{\bm{r}}^*_{L};y)$
		\STATE Backward propagation through network
		\STATE Update $W_k$ and $\bm{b}_k  (k=1,2,\dots,L-1)$
	    \ENDFOR
		\ENDFOR
		\STATE \textbf{\{Stage 2: Fast Inference Training for $f_{ANN}$\}}
		\FOR{e = 1 to $epoch2$}
	    \FOR{length of Dataset $D$}
	    \STATE Sample minibatch $(\bm{x},\bm{y})$ from $D$
	    \STATE $\hat{\bm{r}}^*_0=\hat{\bm{r}}'_0=\bm{x}$
	    \FOR{k = 1 to L}
	    \STATE ${\rm r\_max}_k$ = (1-$m$) ${\rm r\_max}_k$ +$m$  $\max(\hat{\bm{r}}^*_{k-1})$
	    \STATE $\theta_k$ = $p_k$  ${\rm r\_max}_k$
	    \STATE $\hat{\bm{r}}^*_{k}$ = clip[$(W_{k-1} \hat{\bm{r}}^*_k + \bm{b}_{k-1})/\theta_k$,0,1]
	    \ENDFOR
	    \FOR{k = 1 to L}
		\STATE ${\rm r\_max}_k$ = (1-$m$) ${\rm r\_max}_k$ +$m$  $\max(\hat{\bm{r}}'_{k-1}$)
		\STATE $\theta_k$ = $p_k$  ${\rm r\_max}_k$
		\STATE $\hat{\bm{r}}'_{k}$ = clip[$(W_{k-1} \hat{\bm{r}}'_k + \bm{b}_{k-1})/\theta_k$,0,1]
		\ENDFOR
		\STATE Loss = $\mathscr{T}(\hat{\bm{r}}_L^*,\hat{\bm{r}}_L')$ (Eq.\ref{eq:fast_loss})
		\STATE Backward propagation through network
		\STATE Update $p_k  (k=1,2,\dots,L)$
	    \ENDFOR
		\ENDFOR
		\STATE \textbf{\{Converting SNN $f_{SNN}$ from pre-trained $f_{ANN}$\}}
		\FOR{k = 1 to L}
	    \STATE $f_{SNN}.W_{k}=f_{ANN}.W_{k}$
	    \STATE $f_{SNN}.\bm{b}_{k}=f_{ANN}.\bm{b}_{k}$
	    \STATE $f_{SNN}.v_{th,k}=f_{ANN}.\theta_{k}$
	    \ENDFOR
	    \RETURN $f_{SNN}$
	\end{algorithmic}
\end{algorithm}

\section*{Appendix B: Supplementary of Methods and Experiments}

\added{The training and converting algorithm is described in Algorithm~\ref{algo:training}. In Stage 2, since each trainable $p_i$ will scale the output, to reduce the instability of threshold training on the deep model, all layers share the same $p_i$ when training. Besides, to limit $p_i \in [0,1]$ of Rate Norm Layer, use $sigmoid(p'_i)$ in place of $p_i$.}

\added{The experiments are conducted on the PyTorch platform. The GPU used in training is NVIDIA GeForce RTX 2080 Ti.}

\clearpage
\bibliographystyle{named}
\small

\bibliography{ijcai21}

\end{document}